\documentclass{article} 
\usepackage[final]{colm2026_conference}

\usepackage{microtype}
\usepackage{hyperref}
\usepackage{url}
\usepackage{booktabs}

\usepackage{tcolorbox}
\tcbset{colback=white, colframe=black, arc=2pt, boxrule=0.5pt}
\usepackage{enumitem}

\usepackage{array}

\newcolumntype{H}{>{\setbox0=\hbox\bgroup}c<{\egroup}@{}}

\usepackage{graphicx}
\usepackage{subfigure}

\usepackage{amsthm}
\newtheorem{theorem}{Theorem}

\newtheorem{proposition}[theorem]{Proposition}

\usepackage{wrapfig}
\usepackage{xcolor}
\usepackage{listings}
\usepackage{mdframed}
\definecolor{promptbg}{RGB}{245, 245, 250}
\definecolor{promptborder}{RGB}{200, 200, 220}
\usepackage{amssymb}
\usepackage{mathtools}
\usepackage{amsthm}

\usepackage[capitalize,noabbrev]{cleveref}

\theoremstyle{plain}

\theoremstyle{definition}

\theoremstyle{remark}

\newcommand{\R}{\mathbb{R}}
\newcommand{\M}{\mathcal{M}}

\newcommand{\loss}{\mathcal{L}}

\usepackage{lineno}

\definecolor{darkblue}{rgb}{0, 0, 0.5}
\hypersetup{colorlinks=true, citecolor=darkblue, linkcolor=darkblue, urlcolor=darkblue}

\title{FlexRouter: Learning Complementary Model Sets for Flexible LLM Routing}

\author{Wang Wei$^1$, Harry Yang$^2$, Tiankai Yang$^3$, 
Samyadeep Basu$^4$, Hongjie Chen$^5$,\\
\bf{Andy Zhao$^3$, Franck Dernoncourt$^4$, Ryan A. Rossi$^4$, Hoda Eldardiry$^1$\thanks{Corresponding author.}} \\
$^1$Virginia Tech, $^2$Independent Researcher, $^3$University of Southern California, \\$^4$Adobe Research, $^5$Dolby Labs\\
\texttt{\{wangwei718,hdardiry\}@vt.edu}
}

\begin{document}

\ifcolmsubmission
\linenumbers
\fi

\maketitle

\begin{abstract}
Existing Large Language Model (LLM) routing methods score LLMs independently to select top-$k$ models. However, this ignores model correlations and enforces a rigid computational budget. Consequently, routers often select redundant models that share failure modes, limiting the overall probability of success.
To address this, we propose FlexRouter, a routing framework that explicitly models model complementarity. FlexRouter optimizes for \textit{answer coverage}, maximizing the probability that at least one selected model yields a correct response.
This objective aligns with practical inference pipelines where multiple candidate outputs are generated and a downstream verifier or user selects the final one.
We formulate routing as a coverage-oriented subset selection problem and model the routing policy using Determinantal Point Processes (DPPs), which naturally capture both model competence and redundancy. To directly optimize coverage without requiring a ground-truth target subset, we introduce a training objective based on marginalizing over failure sets. During inference, we employ a greedy strategy based on marginal log-determinant gains, enabling the router to adaptively determine subset sizes without a predefined budget.
Extensive experiments on the large-scale RouterEval benchmark demonstrate that our proposed FlexRouter achieves higher coverage with lower redundancy across both in-domain and out-of-domain tasks than strong baselines while maintaining flexible inference cost.
\end{abstract}

\section{Introduction}
\begin{wrapfigure}{r}{0.5\columnwidth}
    \centering
    \vspace{-10pt} 
    \includegraphics[width=0.48\columnwidth]{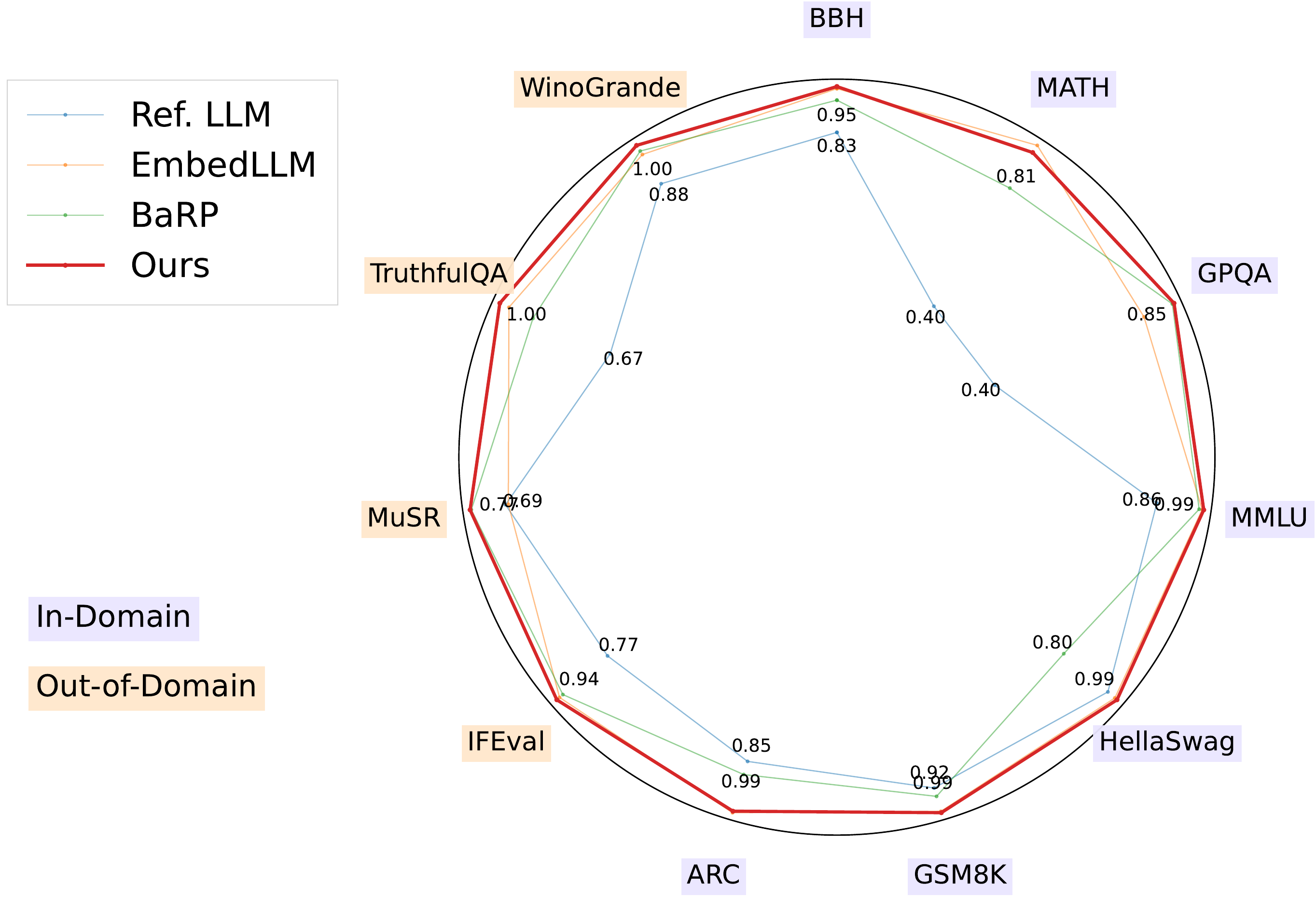}
    \caption{Testing Success@10 score of baselines and FlexRouter on in-domain and out-of-domain tasks.
    }
    \label{fig:allresults}
    \vspace{-10pt} 
\end{wrapfigure}
Large Language Models (LLMs) are increasingly deployed as pools of heterogeneous models that differ in capability, specialization, and inference cost~\citep{chen2023frugalgpt}.
A core problem in using LLMs systematically is \emph{routing}: given an input query, which LLM(s) should be invoked to maximize answer quality under a compute budget.
Routing is closely related to  earlier methods such as classic conditional computation and mixture-of-experts, where a gating mechanism activates a small fraction of experts per input \citep{jacobs1991adaptive,shazeer2017outrageously}.
In modern LLM serving, routing has become especially important since no single model dominates across all tasks \citep{srivatsa2024harnessing}, and the candidate set can be large and rapidly evolving \citep{hu2024routerbench,huang2025routereval}.

\begin{figure*}[t]
\centering
\small

\newlength{\teaserheight}
\setlength{\teaserheight}{5.3cm}

\begin{minipage}[t]{0.323\textwidth}
\begin{tcolorbox}[title=\textbf{Question Answering}, height=\teaserheight, valign=top]
\textbf{Query:} \emph{The car wash is only 100m away from my house, should I walk or drive?}

\vspace{10pt}
\begin{minipage}[t]{0.46\textwidth}
\textbf{Top-$k$}
\begin{itemize}[topsep=3pt, itemsep=3pt, leftmargin=*]
\item Walk...
\item Walk...
\item Walk...
\end{itemize}
\end{minipage}
\hfill
\begin{minipage}[t]{0.52\textwidth}
\textbf{Ours}
\begin{itemize}[topsep=3pt, itemsep=3pt, leftmargin=*]
\item Walk...
\item Drive...
\item Drive...
\end{itemize}
\end{minipage}
\end{tcolorbox}
\end{minipage}
\hfill
\begin{minipage}[t]{0.323\textwidth}
\begin{tcolorbox}[title=\textbf{Math reasoning}, height=\teaserheight, valign=top]
\textbf{Query:} \emph{If you flip a fair coin twice, what is the expected number of heads}

\vspace{10pt}
\begin{minipage}[t]{0.3\textwidth}
\textbf{Top-$k$}
\begin{itemize}[topsep=3pt, itemsep=3pt, leftmargin=*]
\item $1$
\item $1$
\item $1$
\end{itemize}
\end{minipage}
\hfill
\begin{minipage}[t]{0.68\textwidth}
\textbf{Ours}
\begin{itemize}[topsep=3pt, itemsep=3pt, leftmargin=*]
\item 1 by enumerating outcomes
\item 1 by linearity of expectation
\item 1 by symmetry
\end{itemize}
\end{minipage}
\end{tcolorbox}
\end{minipage}
\hfill
\begin{minipage}[t]{0.323\textwidth}
\begin{tcolorbox}[title=\textbf{Coding}, height=\teaserheight, valign=top]
\textbf{Query:} \emph{Write a Python palindrome checker (ignore space and case).}

\vspace{10pt}
\begin{minipage}[t]{0.43\textwidth}
\textbf{Top-$k$}
\begin{itemize}[topsep=3pt, itemsep=3pt, leftmargin=*]
\item s == \newline s[::-1]
\item s == \newline s[::-1]
\item s == \newline s[::-1]
\end{itemize}
\end{minipage}
\hfill
\begin{minipage}[t]{0.55\textwidth}
\textbf{Ours}
\begin{itemize}[topsep=3pt, itemsep=3pt, leftmargin=*]
\item normalize + reverse
\item remove spaces + lowercase
\item s == s[::-1]
\end{itemize}
\end{minipage}
\end{tcolorbox}
\end{minipage}

\caption{
Independent top-$k$ routing often selects highly similar models, leading to redundant outputs and shared failure modes.
FlexRouter instead selects complementary models that cover different interpretations, reasoning paths, and implementations, improving answer coverage across tasks.
}
\label{fig:teaser}
\end{figure*}

A common strategy in LLM routing is to assign each candidate model a predicted score and select the top-$k$ models independently~\citep{embedllm,chen2024routerdc}. Although simple and effective, this design overlooks an important characteristic of modern LLM systems: \emph{model correlation}. Models trained with similar data or architectures often exhibit similar strengths and failure modes. As a result, independent selection of the highest-scoring models can lead to \emph{redundant selections} that produce nearly identical outputs. When these outputs are incorrect, the entire subset fails, limiting the probability of obtaining at least one correct answer. 
As illustrated in Figure~\ref{fig:teaser}, SOTA routing models that do not take diversity into account tend to
produce multiple similar responses that fail in the same way, while our approach that selects a diverse set of models yields varied reasoning paths and increases the likelihood that at least one response is correct. 
Such diversity is particularly important in real-world systems, where multiple candidate responses are generated and a downstream verifier, reranker, or human user selects the final answer. In these settings, the key objective is not that \emph{all} selected models are similarly correct but that they could be diverse and \emph{at least one} of them is correct. This does not by itself solve final answer selection, but it provides a necessary candidate pool for downstream selection. If all selected models fail, no downstream selector can recover the correct answer.

Motivated by this observation, we formulate LLM routing as a \emph{coverage-oriented subset selection} problem. Given an input query, our objective is to select a subset of models that maximizes the likelihood of obtaining at least one correct response while minimizing redundant selections.
To achieve this, we propose \textbf{FlexRouter}, a routing framework that explicitly models both model competence and inter-model correlations. We parameterize the routing policy using Determinantal Point Processes (DPPs)~\citep{han2017faster_map_greedy_dpp}. DPPs naturally define a distribution over subsets, inherently favoring selections that are high-quality and diverse. Unlike independent scoring methods, DPPs penalize the joint selection of highly similar models, encouraging complementary model sets.

A central challenge in this framework is learning under set-valued objectives~\citep{dietterich1997solving_set_value_multi_instance,cour2011learning_partial}. For any given query, there is generally no unique ground-truth subset, as any combination containing at least one successful model is acceptable. Because standard supervised learning requires fixed targets, we must instead optimize directly for answer coverage. To bridge this gap, we introduce a coverage-aligned training objective designed to maximize the probability of sampling a subset with at least one correct model.
This objective is derived by marginalizing over \emph{failure sets}, yielding a tractable objective in terms of DPP determinants.
\begin{figure}[t]
    \centering
    \includegraphics[width=0.9\linewidth]{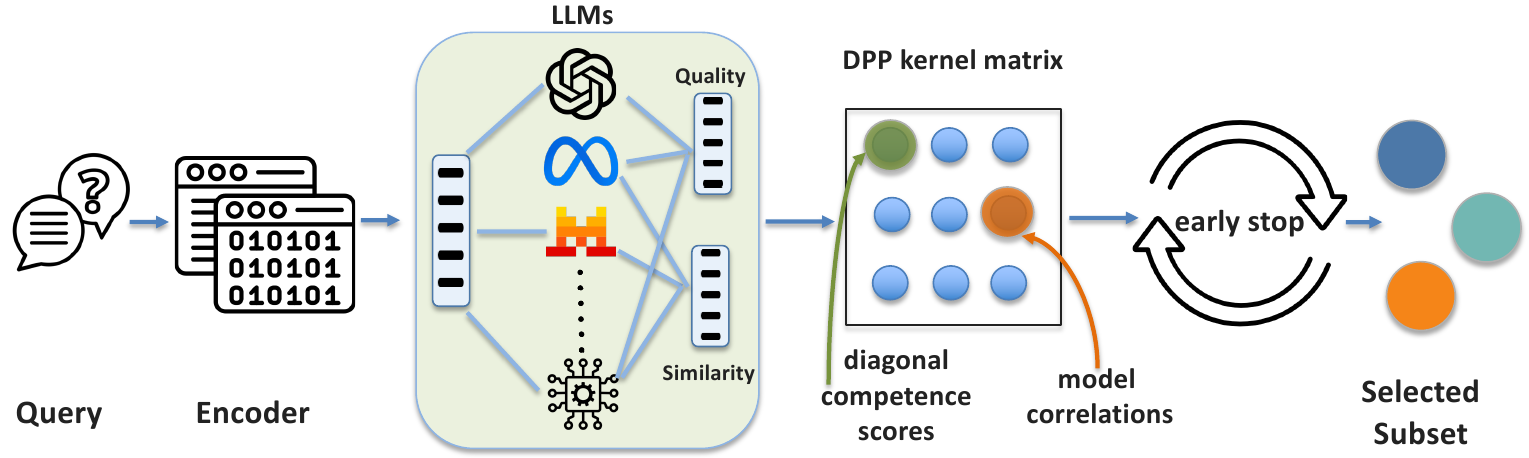}
    \caption{FlexRouter overview. For a query, the router encoder produces embeddings that define query-dependent competence for each LLM, while learned model embeddings capture pairwise similarity between models. These signals define a DPP kernel whose diagonal entries favor strong models and whose off-diagonals penalize redundant selections, thus encouraging individually strong yet mutually diverse model sets. Greedy selection with early stopping returns an adaptive-size subset.}
    \label{fig:placeholder}
\end{figure}
At inference time, FlexRouter performs greedy subset selection using marginal log-determinant gains~\citep{nemhauser1978analysis_submodular} and employs an adaptive stopping rule. This allows the router to dynamically determine the number of models to invoke for each query, allocating more resources to difficult queries while avoiding unnecessary computation on easier ones~\citep{han2021dynamic}.

We evaluate FlexRouter on the large-scale \textsc{RouterEval} benchmark \citep{huang2025routereval}, which provides extensive model--query performance records across a broad range of tasks.
FlexRouter consistently achieves higher answer coverage than strong routing baselines with lower inference cost.
Moreover, FlexRouter reduces redundancy in the selected subsets.
These results highlight the value of explicitly modeling model correlations and performing adaptive subset selection for LLM routing at scale.

\paragraph{Contributions.}
Our main contributions are:
(i) we formulate LLM routing as a \emph{coverage-oriented subset selection} problem that emphasizes the probability of obtaining at least one correct answer;
(ii) we propose \textbf{FlexRouter}, a DPP-based subset routing framework that jointly models model quality and redundancy to select complementary and diverse model sets;
(iii) we derive a tractable \emph{coverage training objective} via failure-set marginalization that directly optimizes the probability of selecting at least one correct model;
(iv) we develop an adaptive greedy inference strategy that selects variable-size subsets without requiring a fixed budget;
(v) we demonstrate consistent improvements on \textsc{RouterEval} \citep{huang2025routereval}, including higher coverage with lower redundancy under flexible inference cost.

\section{Methodology}
\subsection{Problem Formulation}
Let $\M = \{1, \dots, M\}$ be a set of available LLMs. For a given input query $x$, the objective of the router is to select a subset of models $S \subseteq \M$ that maximizes the probability of generating a correct response while minimizing redundancy. 
We assume access to a dataset of query-response pairs for each model, where each query $x$ is associated with a binary label vector ${y} \in \{0, 1\}^M$. Here, $y_i = 1$ if model $i$ answers $x$ correctly, and $y_i = 0$ otherwise.

For a specific query $x$, we partition the model space into a \textit{Correct Set} $C_x$ and a \textit{Failure Set} $F_x$:
\begin{equation}
    C_x = \{i \in \M \mid y_i = 1\}, \quad F_x = \{i \in \M \mid y_i = 0\}.
\end{equation}
The routing task is successful if the selected subset $S$ contains at least one model from $C_x$. Any subset that intersects $C_x$ is considered successful under this objective, as it contains at least one correct model. We define the reward function as the indicator of coverage:
\begin{equation}
    R(x, S) = \mathbb{I}[S \cap C_x \neq \emptyset] = 1 - \prod_{i \in S} (1 - y_i).
\end{equation}

\subsection{Complementary Model Sets Parameterization}
To capture both the individual capabilities of models and their pairwise correlations, we model the subset selection probability using a Determinantal Point Process (DPP) which defines a distribution over subsets that naturally favors high-quality yet non-redundant selections, aligning with the coverage-based routing objective. A DPP is fully parameterized by a positive semi-definite (PSD) $L$-ensemble matrix $L_x \in \R^{M \times M}$. The probability of selecting a subset $S$ is given by:
\begin{equation}
    \mathcal{P}_{L_x}(S) = \frac{\det((L_x)_S)}{\det(I + L_x)},
\end{equation}
where $(L_x)_S$ denotes the submatrix of $L_x$ indexed by the elements of $S$, and $I$ is the identity matrix.
To enable end-to-end learning, we propose a complementary parameterization for the kernel $L_x$. This construction ensures the matrix is inherently PSD and explicitly models the trade-off between model performance and redundancy.

\paragraph{Query and Model Embeddings.} 
Let $E(x) = {v}_x \in \R^d$ be the query embedding generated by a router encoder where $E(\cdot)$ is a shared text encoder that maps the input query into a dense representation. Each model $i$ is assigned a learnable embedding ${u}_i \in \R^d$ representing its functional capabilities.
We compute a scalar quality score $q_i(x) \in \R^+$ for each model, representing the confidence that model $i$ can answer query $x$. This is parameterized as:
\begin{equation}
    q_i(x) = \sigma\left( {v}_x^\top {u}_i \right),
\end{equation}
where $q_i(x)$ captures query-dependent model competence.

We define a similarity matrix ${K} \in \R^{M \times M}$ based on the cosine similarity between model embeddings, which captures the static correlation between models (i.e., models with similar architectures or training data will have high similarity):
\begin{equation}
    K_{ij} = \frac{{u}_i^\top {u}_j}{\|{u}_i\| \|{u}_j\|}.
\end{equation}

\paragraph{Kernel Construction.} 
The final query-dependent kernel $L_x$ is constructed as:
\begin{equation}
    L_x = \operatorname{diag}({q}(x)) \cdot {K} \cdot \operatorname{diag}({q}(x))
\end{equation}
Element-wise, this corresponds to $(L_x)_{ij} = q_i(x) q_j(x) K_{ij}$. The diagonal entries $(L_x)_{ii} = q_i(x)^2$ capture the individual quality of each model, while the off-diagonal entries penalize the joint selection of correlated models. Consequently, the determinant of a subset reflects both the quality and diversity of the selected models, as highly correlated models reduce the determinant.

\subsection{Learn to Maximize Coverage}
The standard maximum likelihood estimate for DPPs assumes access to an observed target subset. In our setting, the supervision does not specify a unique optimal routing decision. Any subset that intersects the correct set $C_x$ is considered successful. As a result, there is no single ground-truth subset to maximize likelihood against. Instead, we propose a Coverage Loss that directly maximizes the probability of obtaining at least one correct answer.
Furthermore, in traditional DPPs, the observed supervision is a subset valued random variable $Y \subseteq \mathcal M$, which enables likelihood based learning via $\log \Pr(Y = S^*)$. In contrast, our setting observes only a binary correctness vector $\mathbf y \in \{0,1\}^M$, which is not a realization of $Y$ and therefore admits no valid DPP maximum likelihood objective.

To efficiently compute the probability of success, we consider its complement, i.e., the probability that a sampled subset $Y$ fails to cover the query (i.e., $Y$ is entirely contained within the failure set $F_x$). This is given by the marginal probability of the failure set (see Appendix~\ref{theory_results}, Proposition~\ref{prop:coverage-diversity} for detailed analysis):
\begin{equation}
    P(Y \subseteq F_x) = \sum_{S \subseteq F_x} \mathcal{P}_{L_x}(S) = \frac{\det(I + (L_x)_{F_x})}{\det(I + L_x)}.
\end{equation}
The probability of success (Hit) is therefore $P_{\text{succ}}(x) = 1 - P(Y \subseteq F_x)$. We minimize the negative log-probability of success as:
\begin{equation}
    \loss_{\text{hit}}(x) = -\log \left( 1 - \frac{\det(I + (L_x)_{F_x})}{\det(I + L_x)} \right).
\end{equation}

\paragraph{Auxiliary Supervision.} 
In addition to the coverage objective, we include a binary cross-entropy (BCE) loss on the predicted model-wise correctness scores $q_i(x)$ to provide direct supervision for individual model predictions. This auxiliary loss stabilizes training by guiding the encoder $E(\cdot)$ and the model embeddings $\{u_i\}$ to produce accurate estimates of per-model correctness. The overall training objective is:
\begin{equation}
\mathcal{L} = \mathcal{L}_{\text{hit}} + \lambda \sum_{i=1}^M \mathrm{BCE}(q_i(x), y_i),
\end{equation}
where $\lambda$ controls the strength of the auxiliary supervision.

\subsection{Flexible Router Inference}
At inference time, we seek the subset $S$ that maximizes the determinant that corresponds to selecting the most probable subset under the learned DPP. This is also known as MAP inference~\citep{gillenwater2012near_optimal_dpp}. However, exact MAP inference is NP-hard~\citep{civril2009selecting_map_np}. We employ a greedy algorithm with an adaptive stopping condition.

We initialize $S = \emptyset$. At each step, we select the model $i \notin S$ that provides the maximal multiplicative gain to the determinant volume. The marginal gain $g_i(S)$ is defined as:
\begin{equation}
    g_i(S) = \frac{\det((L_x)_{S \cup \{i\}})}{\det((L_x)_S)}.
\end{equation}
Using the Schur complement, this can be computed efficiently as:
\begin{equation}
    g_i(S) = (L_x)_{ii} - (L_x)_{i,S} [(L_x)_S]^{-1} (L_x)_{S,i}.
\end{equation}

\textbf{Stopping Condition:}
Unlike top-$k$ routing which enforces a fixed budget, FlexRouter dynamically determines the subset size. We stop adding models when the marginal gain $g_i(S)$ falls below a threshold $\tau \geq 0$ relative to the initial gain. Crucially, because the DPP log-determinant objective is non-monotone, adding highly correlated models can actually decrease the overall subset score. This adaptive stopping rule naturally stops selection when candidates offer no unique contributions (see Appendix~\ref{theory_results}, Proposition~\ref{thm:submodular}-\ref{thm:greedy} for formal proofs and intuition), allowing the router to adaptively balance coverage and computational cost.

\section{Experiments}
We evaluate FlexRouter on large-scale LLM routing benchmarks to answer the following questions:
(i) Does FlexRouter improve coverage over independent-scoring baselines?
(ii) Does modeling complementarity reduce redundancy in the selected model sets?
(iii) Can flexible routing reduce inference cost while maintaining high coverage?
(iv) Does the learned routing policy generalize to unseen tasks?

\subsection{Benchmarks and Baselines}
We evaluate FlexRouter on \textsc{RouterEval}~\citep{routereval}, a large-scale LLM routing benchmark that provides binary correctness labels for each query--model pair. We consider two settings with shared candidate pools: a \textit{medium-pool} setting with 3811 candidate LLMs and a \textit{large-pool} setting with 5000 candidate LLMs.

\textbf{Datasets.} For the medium-pool setting, we train and evaluate in-domain on BBH~\citep{bbh}, MATH~\citep{math}, and GPQA~\citep{gpqa}, and test out-of-domain generalization on IFEval~\citep{ifeval} and MuSR~\citep{musr}. For the large-pool setting, the in-domain tasks are MMLU~\citep{hendrycks2021mmlu}, HellaSwag~\citep{zellers2019hellaswag}, GSM8K~\citep{cobbe2021gsm8k}, and ARC~\citep{clark2018arc}, while TruthfulQA~\citep{truthfulqa} and WinoGrande~\citep{sakaguchi2021winogrande} are used for out-of-domain evaluation. For in-domain evaluation, we report results on held-out 20\% test splits of the training tasks. For out-of-domain evaluation, the router is trained only on the in-domain tasks and evaluated on unseen tasks. Details are in Table~\ref{tab:routereval_datasets} and Appendix~\ref{app:data}.

\textbf{Baselines.}
We compare FlexRouter with strong routing baselines. \textbf{Independent scoring} ranks models by predicted correctness and selects the top-$k$ candidates independently, without modeling inter-model correlations. \textbf{EmbedLLM}~\citep{embedllm} learns query-dependent model scores in a shared embedding space, but still performs independent top-$k$ selection. \textbf{BaRP}~\citep{barp} formulates routing as a multi-objective contextual bandit and learns an adaptive policy under bandit feedback. We also report the \textbf{Ref.\ score}, which is the performance of a representative strong single model on each benchmark and serves as a single-model reference. Additional baseline details are deferred to Appendix~\ref{app:baseline}. For fixed-budget comparisons, all routing methods are evaluated with the same maximum selection budget.

We also evaluate additional coverage- and diversity-oriented baselines, including Random-k, MaxDiversity, and MMR. We defer their definitions and results to Appendix~\ref{app:coverage_baselines}.

\subsection{Evaluation Metrics}

Let $\mathcal{Q}$ denote the set of queries. We evaluate routing quality using \textbf{Success@$k$}, the fraction of queries for which at least one selected model is correct:
\begin{equation}
  \text{Success@}k
  \;=\;
  \frac{1}{|\mathcal{Q}|}\sum_{q \in \mathcal{Q}}
  \mathbf{1}\!\left[\,\exists\, m \in \mathcal S_q : y_{q,m} = 1 \right].
  \label{eq:success_at_k}
\end{equation}

To measure redundancy, we report \textbf{ILD@$k$} (Intra-List Diversity), defined as the average pairwise cosine distance among the selected models in a fixed pretrained embedding space:
\begin{equation}
  \text{ILD@}k
  \;=\;
  \frac{1}{|\mathcal{Q}|}\sum_{q \in \mathcal{Q}}
  \frac{1}{\binom{|\mathcal S_q|}{2}}
  \sum_{\substack{i,j \,\in\, \mathcal S_q \\ i < j}}
  \left(1 - \frac{\mathbf{u}_i^\top \mathbf{u}_j}{\|\mathbf{u}_i\|\,\|\mathbf{u}_j\|}\right).
  \label{eq:ild_at_k}
\end{equation}
Higher ILD indicates less redundant and more complementary selections. Further metric details are provided in Appendix~\ref{app:metrics}.

\subsection{Implementation Details}
FlexRouter uses a shared query encoder to map each query into a representation, and learnable model embeddings to represent candidate models. Queries are encoded using RoBERTa-base and projected into a shared 128-dimensional space, which is used to construct the query-dependent DPP kernel. The model is trained with Adam for up to 100 epochs. The training objective combines the coverage loss with an auxiliary binary cross-entropy loss with weight $\lambda = 1.0$. At inference time, model selection is performed using a greedy MAP procedure based on marginal log-determinant gains, with adaptive stopping and a maximum subset size of $k=10$. All experiments are conducted on NVIDIA A100 80GB GPUs. Details and computational complexity are provided in Appendix~\ref{app:imp}.

\textbf{Downstream Answer Selection.}
FlexRouter focuses on candidate-pool construction rather than final answer selection. Its objective is to select a complementary subset of models such that at least one selected model is likely to produce a correct response for the given task. 

Therefore, Success@k should be interpreted as a routing-stage coverage metric, not as final deployed accuracy. In practice, the routed candidates can be passed to a task-specific verifier, reward model, LLM-as-judge reranker, self-consistency or aggregation method, or human user. However, these downstream mechanisms are not guaranteed to recover a lone correct response in every setting. A full end-to-end evaluation with a concrete selector is therefore complementary to our work and remains an important future direction.

\subsection{Routing Performance}
\label{sec:main_results}
\begin{table*}[t]
\centering
\small
\caption{Success@10 on the \textbf{medium-pool setting} (3811 models). Ref. score is the reference average accuracy for each task of a representative LLM, such as GPT-4. Best results are \textbf{bolded}; second-best are \underline{underlined}.
}
\label{tab:succ10-new}
\begin{tabular}{lccc cc c}
\toprule
 & \multicolumn{3}{c}{\textbf{In-Domain}} & \multicolumn{2}{c}{\textbf{Out-of-Domain}} &  \\
\cmidrule(lr){2-4} \cmidrule(lr){5-6}
Method & BBH & MATH & GPQA & IFEval & MuSR & Avg \\
\midrule
Ref. score   & 0.830 & 0.400 & 0.397 & 0.769 & 0.699 & 0.619 \\
EmbedLLM    & \underline{0.9419} & \textbf{0.8264} & 0.7731 & \underline{0.9298} & 0.6931 & \underline{0.8329} \\
BaRP        & 0.9124 & 0.7132 & \underline{0.8445} & 0.9187 & \underline{0.7712} & 0.8320 \\
\textbf{Ours} & \textbf{0.9471} & \underline{0.8075} & \textbf{0.8487} & \textbf{0.9390} & \textbf{0.7738} & \textbf{0.8632} \\
\bottomrule
\end{tabular}
\end{table*}

\begin{table*}[t]
\centering
\small
\caption{Success@10 on the \textbf{large-pool setting} (5000 models). Best results are \textbf{bolded}; second-best are \underline{underlined}.
}
\label{tab:succ10-old}
\begin{tabular}{lcccc cc c}
\toprule
 & \multicolumn{4}{c}{\textbf{In-Domain}} & \multicolumn{2}{c}{\textbf{Out-of-Domain}} &  \\
\cmidrule(lr){2-5} \cmidrule(lr){6-7}
Method & MMLU & HellaSwag & GSM8K & ARC & TruthfulQA & WinoGrande & Avg \\
\midrule
Ref. score   & 0.864 & 0.953 & 0.920 & 0.852 & 0.669 & 0.875 & 0.855 \\
EmbedLLM    & \underline{0.9865} & \underline{0.9776} & \underline{0.9848} & \textbf{0.9957} & \underline{0.9682} & 0.9684 & \underline{0.9802} \\
BaRP        & 0.9783 & 0.7979 & 0.9432 & 0.8889 & 0.8972 & \underline{0.9795} & 0.9142 \\
\textbf{Ours} & \textbf{0.9907} & \textbf{0.9851} & \textbf{0.9886} & \underline{0.9915} & \textbf{0.9951} & \textbf{0.9976} & \textbf{0.9914} \\
\bottomrule
\end{tabular}
\end{table*}

Tables~\ref{tab:succ10-new} and~\ref{tab:succ10-old} report routing performance
measured by Success@10 on the two RouterEval settings.
Across both settings, FlexRouter achieves the highest average success rate,
demonstrating the benefit of selecting complementary model subsets.

\textbf{Medium-pool setting results.}
The medium-pool setting (Table~\ref{tab:succ10-new}) contains fewer candidate models and several challenging
reasoning benchmarks.
FlexRouter achieves the best overall performance with an average
Success@10 of \textbf{0.8632}.
The improvement is most pronounced on GPQA, where FlexRouter substantially
outperforms both baselines.
The method also achieves the highest performance on both out-of-domain
tasks, suggesting that the learned routing policy
generalizes well to unseen tasks.

\textbf{Large-pool setting results.}
On the large-pool setting (Table~\ref{tab:succ10-old}), FlexRouter achieves the best
overall performance with an average Success@10 of \textbf{0.9914}.
The method obtains the highest score on five of the six tasks and performs
particularly well on the out-of-domain benchmarks TruthfulQA and WinoGrande.
These results indicate that explicitly modeling correlations between models
helps identify complementary candidates and increases the likelihood that at
least one selected model produces a correct answer.

EmbedLLM achieves the highest score on MATH.
This behavior likely reflects that MATH contains a small number of specialist
models with strong performance, where selecting the highest-scoring models
independently can already perform well.
Nevertheless, FlexRouter consistently achieves the best average performance
across tasks, indicating that modeling model complementarity provides a
robust advantage across diverse benchmarks.

\textbf{Additional coverage analysis.}
To further analyze the behavior across smaller budgets and the composition of selected pools, Appendix~\ref{app:coverage_analysis} reports Success@k curves, Avg-Correct@10, and Zero-Correct Rate. These analyses show that FlexRouter's advantage appears in the multi-model setting and that it reduces all-wrong selected pools.

\subsection{Subset Diversity}
\label{sec:diversity}

\begin{table*}[t]
\centering
\small
\caption{ILD@10 in the fixed pre-trained embedding space on the \textbf{medium-pool setting}.
  Higher is better. 
  Best results are \textbf{bolded}.
}
\label{tab:ild10-new}
\begin{tabular}{lccc cc c}
\toprule
 & \multicolumn{3}{c}{\textbf{In-Domain}} & \multicolumn{2}{c}{\textbf{Out-of-Domain}} &  \\
\cmidrule(lr){2-4} \cmidrule(lr){5-6}
Method & BBH & MATH & GPQA & IFEval & MuSR & Avg \\
\midrule
EmbedLLM    & 0.2114    & 0.1432    & 0.2196    & 0.2524    & 0.2295    & 0.2112 \\
BaRP        & 0.1983 & 0.1993 & 0.1990 & 0.1966 & 0.1980 & 0.1982 \\
\textbf{Ours} & \textbf{0.7249} & \textbf{0.7515} & \textbf{0.7413} & \textbf{0.7895} & \textbf{0.8027} & \textbf{0.7620} \\
\bottomrule
\end{tabular}
\end{table*}

\begin{table*}[t]
\centering
\small
\caption{ILD@10 in the fixed pre-trained embedding space on the \textbf{large-pool setting}.
  Higher is better.  Best results are \textbf{bolded}.
}
\label{tab:ild10-old}
\begin{tabular}{lcccc cc c}
\toprule
 & \multicolumn{4}{c}{\textbf{In-Domain}} & \multicolumn{2}{c}{\textbf{Out-of-Domain}} &  \\
\cmidrule(lr){2-5} \cmidrule(lr){6-7}
Method & MMLU & HellaSwag & GSM8K & ARC & TruthfulQA & WinoGrande & Avg \\
\midrule
EmbedLLM    & 0.1302 & 0.2765 & 0.1775 & 0.4536 & 0.3747 & 0.2732 & 0.2810 \\
BaRP        & \textbf{0.5520} & 0.5521 & 0.5521 & 0.5520 & 0.5521 & 0.5521 & 0.5521 \\
\textbf{Ours} & 0.4709 & \textbf{0.7052} & \textbf{0.6483} & \textbf{0.7653} & \textbf{0.7396} & \textbf{0.7516} & \textbf{0.6802} \\
\bottomrule
\end{tabular}
\end{table*}

We next examine whether the selected model subsets exhibit higher diversity.
Tables~\ref{tab:ild10-new} and~\ref{tab:ild10-old} report ILD@10, which measures
the average pairwise cosine distance between embeddings of the selected models.
Higher values indicate that the router selects models that are less redundant
and more complementary.

Across both settings, FlexRouter consistently produces the most diverse model
subsets.
On the \textbf{large-pool setting} (Table~\ref{tab:ild10-old}), FlexRouter achieves the
highest average ILD@10 and obtains the best result on five of the six tasks.
In particular, the diversity gains are most pronounced on ARC, TruthfulQA, and
WinoGrande, where the selected models are substantially more dissimilar than
those chosen by the baselines.
Although BaRP produces relatively high diversity score on MMLU, its selections are
nearly constant across tasks, suggesting that it relies on a fixed set of
models rather than adapting to query-specific complementarity.

The difference is even more pronounced on the \textbf{medium-pool setting}
(Table~\ref{tab:ild10-new}), where FlexRouter achieves substantially higher
ILD@10 on every task.
In contrast, EmbedLLM and BaRP produce relatively low diversity scores,
indicating that their routing strategies tend to select similar models.

These results provide empirical evidence that FlexRouter effectively captures
model complementarity during routing.
By explicitly modeling correlations between models, the router selects
subsets that are both high-quality and diverse, which directly supports the
coverage improvements observed in Section~\ref{sec:main_results}.

\subsection{Generalization to Unseen Tasks}
\label{sec:generalization}

We next examine whether the learned routing policy generalizes to tasks that are not observed during training.
Across both RouterEval settings, FlexRouter consistently achieves the best Success@10 on the out-of-domain benchmarks according to Tables~\ref{tab:succ10-new} and~\ref{tab:succ10-old}, indicating that the routing strategy learned on the in-domain tasks transfers effectively to unseen evaluation settings.
The diversity analysis further supports this observation.
As shown in Tables~\ref{tab:ild10-new} and~\ref{tab:ild10-old}, FlexRouter selects substantially more diverse model subsets on the out-of-domain tasks compared with both baselines.
This behavior suggests that the router captures general patterns of model complementarity rather than memorizing task-specific preferences.
Together, these results indicate that modeling correlations between models enables FlexRouter to construct complementary model subsets that remain effective even when routing queries from previously unseen tasks.

\subsection{Ablation on Kernel Design}
\label{sec:kernel_ablation}

The DPP kernel requires a similarity measure between model embeddings.
Our default design uses \textbf{cosine similarity}, while an alternative
is a \textbf{Gaussian RBF kernel} (as detailed in Appendix~\ref{app:rbf})
$S_{ij} = \exp(-\gamma \|\mathbf{u}_i - \mathbf{u}_j\|^2)$
with a trainable bandwidth $\gamma$.
Table~\ref{tab:kernel_ablation} reports results on both RouterEval settings
at $k=10$. On the large-pool setting, the RBF kernel achieves slightly higher
in-domain Success@10, while cosine performs better on the out-of-domain
tasks and consistently produces more diverse model subsets.
On the medium-pool setting, cosine achieves higher in-domain success and again
selects more diverse subsets, while RBF attains slightly higher
out-of-domain success. Overall, cosine similarity consistently yields higher diversity while
maintaining competitive success rates across both settings.
We therefore adopt cosine similarity as the default kernel for FlexRouter.

\begin{table}[t]
\centering
\small
\caption{Kernel design ablation at $k = 10$.
  ILD is measured in the fixed embedding space.}
\label{tab:kernel_ablation}
\setlength{\tabcolsep}{5pt}
\begin{tabular}{l l c c c c H}
\toprule
Setting & Kernel & ID Succ & ID ILD & OOD Succ & OOD ILD & $\gamma$ \\
\midrule
{Large-pool}
  & RBF      & \textbf{0.9920} & 0.6019 & 0.9954 & 0.6923 & 2.19 \\
  & \textbf{Cosine} & 0.9854 & \textbf{0.6474} & \textbf{0.9962} & \textbf{0.7456} & N/A \\
\midrule
{Medium-pool}
  & RBF      & 0.8804 & 0.7109 & \textbf{0.8757} & 0.7883 & 1.56 \\
  & \textbf{Cosine} & \textbf{0.9058} & \textbf{0.7392} & 0.8486 & \textbf{0.7961} & N/A \\
\bottomrule
\end{tabular}
\end{table}

\subsection{Adaptive Subset Size via Stopping Threshold $\tau$}
\label{sec:tau}

FlexRouter allows adaptive control of inference cost through a stopping threshold $\tau$ in the greedy MAP procedure.
A model is added only if its marginal log-determinant gain exceeds $\tau$ times the first gain, allowing the router to stop early when additional models provide limited benefit.
Figure~\ref{fig:tau} illustrates the resulting coverage–cost trade-off on the combined in-domain test set of the large-pool setting ($k_{\max}=10$). When $\tau$ is small, the router selects nearly all candidate models and achieves the highest coverage. As $\tau$ increases, the average subset size decreases while Success@10 gradually declines.
A favorable operating point appears around $\tau \approx 0.2$,
where the router reduces the average subset size substantially
while maintaining high coverage. Beyond this point, further reductions in subset size lead to a sharper drop in coverage. Exact numerical results are reported in Table~\ref{tab:tau_ablation}.

\begin{figure}
    \centering
    \small
    \includegraphics[width=\linewidth]{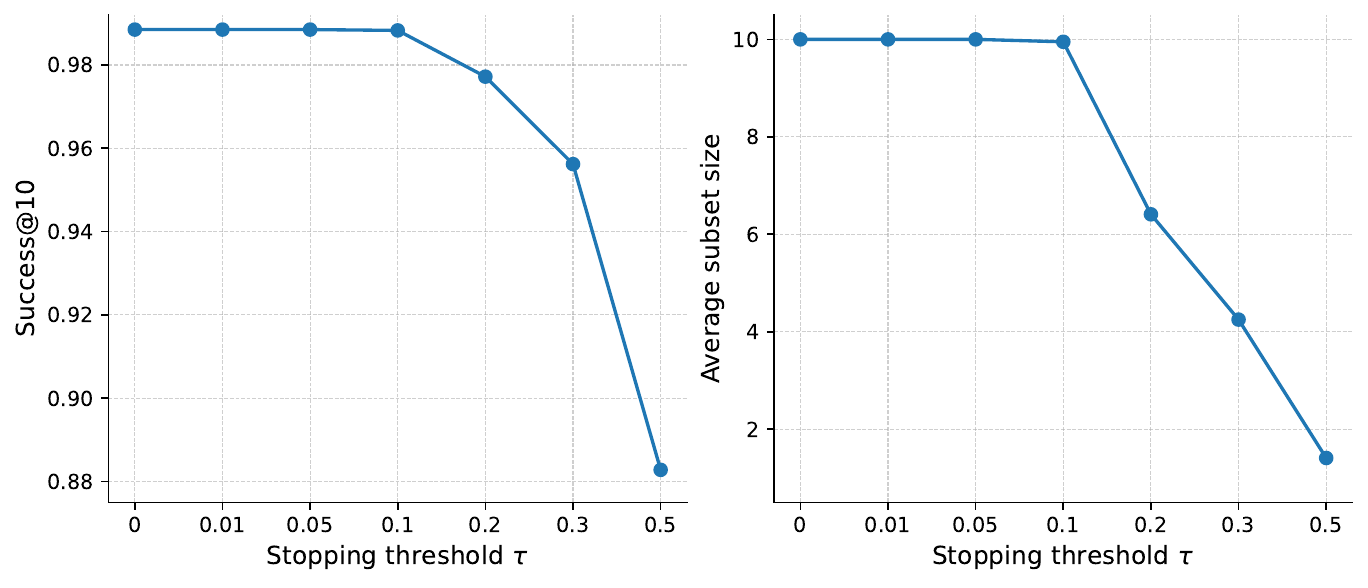}
    \caption{Success@10 and Average Subset Size over Stopping Threshold $\tau$. When $\tau \approx 0.2$, the router reduces the average subset size substantially while maintaining high coverage.}
    \label{fig:tau}
\end{figure}

These results demonstrate that the stopping threshold provides a simple mechanism for controlling the coverage–cost trade-off. Importantly, the subset size varies across queries, indicating that FlexRouter adapts its inference budget to query difficulty rather than using a fixed number of models.

\section{Limitations and Scope}
\label{sec:limitations}

FlexRouter focuses on routing-stage candidate-pool construction rather than final answer selection. Success@k measures whether the selected subset contains at least one correct model, so it should not be interpreted as final deployed accuracy. The selected candidates may be passed to a verifier, reranker, aggregation method, or human user, but these downstream mechanisms are not guaranteed to recover a correct answer in every case. Jointly evaluating FlexRouter with concrete downstream selectors is an important future direction.

Our diversity analysis measures model-level diversity using fixed model embeddings, which does not necessarily imply output-level semantic diversity for every query. Since RouterEval provides correctness labels but not full generated responses for all query--model pairs, response-level diversity and answer-agreement analysis are left for future work.

Our cost analysis uses the number of invoked models, or average subset size, as a proxy for inference cost. Real deployment cost also depends on model-specific latency, token price, output length, batching, and hardware constraints. Extending FlexRouter to optimize such model-specific costs is a useful future direction.

\section{Conclusion}\label{sec:conc}
We presented \textbf{FlexRouter}, a coverage-oriented routing framework that selects complementary subsets of LLMs under flexible inference budgets. By modeling routing as a subset selection problem with determinantal point processes, FlexRouter explicitly accounts for model correlations and reduces redundant model selections. A coverage-aligned training objective enables learning without requiring a ground-truth target subset while an adaptive greedy inference procedure determines subset sizes on a per-query basis without enforcing a fixed budget. 
Experiments and ablation studies on the \textsc{RouterEval} benchmark demonstrate that FlexRouter achieves higher coverage than strong baselines while maintaining flexible inference cost and lower redundancy, highlighting the benefits of correlation-aware and adaptive subset selection for LLM routing at scale.

\section*{Ethics Statement}
This work studies routing among existing LLMs and does not introduce new models or training data. As such, it inherits potential biases and risks present in the underlying models. By promoting diverse model selection, our method may surface a broader range of outputs, including both useful alternatives and undesirable content. In practice, it should be combined with downstream safeguards such as verification, filtering, or human oversight. Our approach may improve efficiency by reducing redundant model usage, but adaptive routing can also increase cost for difficult queries if not properly controlled.

\bibliography{colm2026_conference}

@misc{barp,
      title={Learning to Route LLMs from Bandit Feedback: One Policy, Many Trade-offs}, 
      author={Wang Wei and Tiankai Yang and Hongjie Chen and Yue Zhao and Franck Dernoncourt and Ryan A. Rossi and Hoda Eldardiry},
      year={2025},
      eprint={2510.07429},
      archivePrefix={arXiv},
      primaryClass={cs.LG},
      url={https://arxiv.org/abs/2510.07429}, 
}

@inproceedings{
embedllm,
title={Embed{LLM}: Learning Compact Representations of Large Language Models},
author={Richard Zhuang and Tianhao Wu and Zhaojin Wen and Andrew Li and Jiantao Jiao and Kannan Ramchandran},
booktitle={The Thirteenth International Conference on Learning Representations},
year={2025},
url={https://openreview.net/forum?id=Fs9EabmQrJ}
}

@inproceedings{routereval,
    title = "{R}outer{E}val: A Comprehensive Benchmark for Routing {LLM}s to Explore Model-level Scaling Up in {LLM}s",
    author = "Huang, Zhongzhan  and
      Ling, Guoming  and
      Lin, Yupei  and
      Chen, Yandong  and
      Zhong, Shanshan  and
      Wu, Hefeng  and
      Lin, Liang",
    editor = "Christodoulopoulos, Christos  and
      Chakraborty, Tanmoy  and
      Rose, Carolyn  and
      Peng, Violet",
    booktitle = "Findings of the Association for Computational Linguistics: EMNLP 2025",
    month = nov,
    year = "2025",
    address = "Suzhou, China",
    publisher = "Association for Computational Linguistics",
    url = "https://aclanthology.org/2025.findings-emnlp.208/",
    doi = "10.18653/v1/2025.findings-emnlp.208",
    pages = "3860--3887",
    ISBN = "979-8-89176-335-7"
}

@inproceedings{huang2025routereval,
  title     = {RouterEval: A Comprehensive Benchmark for Routing {LLM}s to Explore Model-level Scaling Up in {LLM}s},
  author    = {Huang, Zhongzhan and Ling, Guoming and Liang, Vincent S. and Lin, Yupei and Chen, Yandong and Zhong, Shanshan and Wu, Hefeng and Lin, Liang},
  booktitle = {Findings of the Association for Computational Linguistics: {EMNLP} 2025},
  year      = {2025}
}

@article{hu2024routerbench,
  title   = {RouterBench: A Benchmark for Multi-{LLM} Routing System},
  author  = {Hu, Qitian Jason and Bieker, Jacob and Li, Xiuyu and Jiang, Nan and Keigwin, Benjamin and Ranganath, Gaurav and Keutzer, Kurt and Upadhyay, Shriyash Kaustubh},
  journal = {arXiv preprint arXiv:2403.12031},
  year    = {2024}
}

@inproceedings{chen2024routerdc,
  title     = {RouterDC: Query-Based Router by Dual Contrastive Learning for Assembling Large Language Models},
  author    = {Chen, Shuhao and Jiang, Weisen and Lin, Baijiong and Kwok, James T. and Zhang, Yu},
  booktitle = {Advances in Neural Information Processing Systems},
  year      = {2024}
}

@inproceedings{srivatsa2024harnessing,
  title     = {Harnessing the Power of Multiple Minds: Lessons Learned from {LLM} Routing},
  author    = {Srivatsa, Kv Aditya and Maurya, Kaushal Kumar and Kochmar, Ekaterina},
  booktitle = {Proceedings of the Fifth Workshop on Insights from Negative Results in {NLP}},
  year      = {2024},
  doi       = {10.18653/v1/2024.insights-1.15}
}

@article{kulesza2012dppfnt,
  title   = {Determinantal Point Processes for Machine Learning},
  author  = {Kulesza, Alex and Taskar, Ben},
  journal = {Foundations and Trends{\textregistered} in Machine Learning},
  year    = {2012}
}

@article{jacobs1991adaptive,
  title   = {Adaptive Mixtures of Local Experts},
  author  = {Jacobs, Robert A. and Jordan, Michael I. and Nowlan, Steven J. and Hinton, Geoffrey E.},
  journal = {Neural Computation},
  volume  = {3},
  number  = {1},
  pages   = {79--87},
  year    = {1991},
  doi     = {10.1162/neco.1991.3.1.79}
}

@inproceedings{shazeer2017outrageously,
  title     = {Outrageously Large Neural Networks: The Sparsely-Gated Mixture-of-Experts Layer},
  author    = {Shazeer, Noam and Mirhoseini, Azalia and Maziarz, Krzysztof and Davis, Andy and Le, Quoc and Hinton, Geoffrey and Dean, Jeff},
  booktitle = {International Conference on Learning Representations ({ICLR})},
  year      = {2017},
  url       = {https://arxiv.org/abs/1701.06538}
}

@article{cobbe2021gsm8k,
  title={Training verifiers to solve math word problems},
  author={Cobbe, Karl and Kosaraju, Vineet and Bavarian, Mohammad and Chen, Mark and Jun, Heewoo and Kaiser, Lukasz and Plappert, Matthias and Tworek, Jerry and Hilton, Jacob and Nakano, Reiichiro and others},
  journal={arXiv preprint arXiv:2110.14168},
  year={2021}
}

@article{hendrycks2021mmlu,
  title={Measuring massive multitask language understanding},
  author={Hendrycks, Dan and Burns, Collin and Basart, Steven and Zou, Andy and Mazeika, Mantas and Song, Dawn and Steinhardt, Jacob},
  journal={arXiv preprint arXiv:2009.03300},
  year={2021}
}

@article{clark2018arc,
  title={Think you have solved question answering? try arc, the ai2 reasoning challenge},
  author={Clark, Peter and Cowhey, Isaac and Etzioni, Oren and Khot, Tushar and Sabharwal, Ashish and Schoenick, Carissa and Tafjord, Oyvind},
  journal={arXiv preprint arXiv:1803.05457},
  year={2018}
}

@article{sakaguchi2021winogrande,
  title={WinoGrande: An Adversarial Winograd Schema Challenge at Scale},
  author={Sakaguchi, Keisuke and Le Bras, Ronan and Bhagavatula, Chandra and Choi, Yejin},
  journal={Communications of the ACM},
  year={2021}
}

@article{zellers2019hellaswag,
  title={Hellaswag: Can a machine really finish your sentence?},
  author={Zellers, Rowan and Holtzman, Ari and Bisk, Yonatan and Farhadi, Ali and Choi, Yejin},
  journal={arXiv preprint arXiv:1905.07830},
  year={2019}
}

@misc{truthfulqa,
      title={TruthfulQA: Measuring How Models Mimic Human Falsehoods}, 
      author={Stephanie Lin and Jacob Hilton and Owain Evans},
      year={2022},
      eprint={2109.07958},
      archivePrefix={arXiv},
      primaryClass={cs.CL},
      url={https://arxiv.org/abs/2109.07958}, 
}

@misc{musr,
      title={MuSR: Testing the Limits of Chain-of-thought with Multistep Soft Reasoning}, 
      author={Zayne Sprague and Xi Ye and Kaj Bostrom and Swarat Chaudhuri and Greg Durrett},
      year={2024},
      eprint={2310.16049},
      archivePrefix={arXiv},
      primaryClass={cs.CL},
      url={https://arxiv.org/abs/2310.16049}, 
}

@misc{ifeval,
      title={Instruction-Following Evaluation for Large Language Models}, 
      author={Jeffrey Zhou and Tianjian Lu and Swaroop Mishra and Siddhartha Brahma and Sujoy Basu and Yi Luan and Denny Zhou and Le Hou},
      year={2023},
      eprint={2311.07911},
      archivePrefix={arXiv},
      primaryClass={cs.CL},
      url={https://arxiv.org/abs/2311.07911}, 
}

@misc{gpqa,
      title={GPQA: A Graduate-Level Google-Proof Q\&A Benchmark}, 
      author={David Rein and Betty Li Hou and Asa Cooper Stickland and Jackson Petty and Richard Yuanzhe Pang and Julien Dirani and Julian Michael and Samuel R. Bowman},
      year={2023},
      eprint={2311.12022},
      archivePrefix={arXiv},
      primaryClass={cs.AI},
      url={https://arxiv.org/abs/2311.12022}, 
}

@misc{math,
      title={Measuring Mathematical Problem Solving With the MATH Dataset}, 
      author={Dan Hendrycks and Collin Burns and Saurav Kadavath and Akul Arora and Steven Basart and Eric Tang and Dawn Song and Jacob Steinhardt},
      year={2021},
      eprint={2103.03874},
      archivePrefix={arXiv},
      primaryClass={cs.LG},
      url={https://arxiv.org/abs/2103.03874}, 
}

@article{bbh,
  title={Challenging BIG-Bench Tasks and Whether Chain-of-Thought Can Solve Them},
  author={Suzgun, Mirac and Scales, Nathan and Sch{\"a}rli, Nathanael and Gehrmann, Sebastian and Tay, Yi and Chung, Hyung Won and Chowdhery, Aakanksha and Le, Quoc V and Chi, Ed H and Zhou, Denny and and Wei, Jason},
  journal={arXiv preprint arXiv:2210.09261},
  year={2022}
}

@misc{gpt4,
      title={GPT-4 Technical Report}, 
      author={OpenAI and Josh Achiam and Steven Adler and Sandhini Agarwal and others},
      year={2024},
      eprint={2303.08774},
      archivePrefix={arXiv},
      primaryClass={cs.CL},
      url={https://arxiv.org/abs/2303.08774}, 
}

@inproceedings{han2017faster_map_greedy_dpp,
  title={Faster greedy MAP inference for determinantal point processes},
  author={Han, Insu and Kambadur, Prabhanjan and Park, Kyoungsoo and Shin, Jinwoo},
  booktitle={International Conference on Machine Learning},
  pages={1384--1393},
  year={2017},
  organization={PMLR}
}

@article{chen2023frugalgpt,
  title={Frugalgpt: How to use large language models while reducing cost and improving performance},
  author={Chen, Lingjiao and Zaharia, Matei and Zou, James},
  journal={arXiv preprint arXiv:2305.05176},
  year={2023}
}

@inproceedings{jiang2023llm_blender,
  title={Llm-blender: Ensembling large language models with pairwise ranking and generative fusion},
  author={Jiang, Dongfu and Ren, Xiang and Lin, Bill Yuchen},
  booktitle={Proceedings of the 61st Annual Meeting of the Association for Computational Linguistics (Volume 1: Long Papers)},
  pages={14165--14178},
  year={2023}
}

@article{wang2022self_consistency,
  title={Self-consistency improves chain of thought reasoning in language models},
  author={Wang, Xuezhi and Wei, Jason and Schuurmans, Dale and Le, Quoc and Chi, Ed and Narang, Sharan and Chowdhery, Aakanksha and Zhou, Denny},
  journal={arXiv preprint arXiv:2203.11171},
  year={2022}
}

@article{wang2024mixture_of_agents,
  title={Mixture-of-agents enhances large language model capabilities},
  author={Wang, Junlin and Wang, Jue and Athiwaratkun, Ben and Zhang, Ce and Zou, James},
  journal={arXiv preprint arXiv:2406.04692},
  year={2024}
}

@article{gillenwater2012near_optimal_dpp,
  title={Near-optimal map inference for determinantal point processes},
  author={Gillenwater, Jennifer and Kulesza, Alex and Taskar, Ben},
  journal={Advances in Neural Information Processing Systems},
  volume={25},
  year={2012}
}

@article{civril2009selecting_map_np,
  title={On selecting a maximum volume sub-matrix of a matrix and related problems},
  author={Civril, Ali and Magdon-Ismail, Malik},
  journal={Theoretical Computer Science},
  volume={410},
  number={47-49},
  pages={4801--4811},
  year={2009},
  publisher={Elsevier}
}

@inproceedings{affandi2014learning_information_retrieval,
  title={Learning the parameters of determinantal point process kernels},
  author={Affandi, Raja Hafiz and Fox, Emily and Adams, Ryan and Taskar, Ben},
  booktitle={International Conference on Machine Learning},
  pages={1224--1232},
  year={2014},
  organization={PMLR}
}

@inproceedings{deng2020personalized_information_retrieval,
  title={Personalized bundle recommendation in online games},
  author={Deng, Qilin and Wang, Kai and Zhao, Minghao and Zou, Zhene and Wu, Runze and Tao, Jianrong and Fan, Changjie and Chen, Liang},
  booktitle={Proceedings of the 29th ACM international conference on information \& knowledge management},
  pages={2381--2388},
  year={2020}
}

@inproceedings{wilhelm2018practical_recommendation,
  title={Practical diversified recommendations on youtube with determinantal point processes},
  author={Wilhelm, Mark and Ramanathan, Ajith and Bonomo, Alexander and Jain, Sagar and Chi, Ed H and Gillenwater, Jennifer},
  booktitle={Proceedings of the 27th ACM International Conference on Information and Knowledge Management},
  pages={2165--2173},
  year={2018}
}

@inproceedings{cho2019multi_document,
  title={Multi-document summarization with determinantal point processes and contextualized representations},
  author={Cho, Sangwoo and Li, Chen and Yu, Dong and Foroosh, Hassan and Liu, Fei},
  booktitle={Proceedings of the 2nd Workshop on New Frontiers in Summarization},
  pages={98--103},
  year={2019}
}

@article{han2021dynamic,
  title={Dynamic neural networks: A survey},
  author={Han, Yizeng and Huang, Gao and Song, Shiji and Yang, Le and Wang, Honghui and Wang, Yulin},
  journal={IEEE transactions on pattern analysis and machine intelligence},
  volume={44},
  number={11},
  pages={7436--7456},
  year={2021},
  publisher={IEEE}
}

@article{nemhauser1978analysis_submodular,
  title={An analysis of approximations for maximizing submodular set functions—I},
  author={Nemhauser, George L and Wolsey, Laurence A and Fisher, Marshall L},
  journal={Mathematical programming},
  volume={14},
  number={1},
  pages={265--294},
  year={1978},
  publisher={Springer}
}

@article{dietterich1997solving_set_value_multi_instance,
  title={Solving the multiple instance problem with axis-parallel rectangles},
  author={Dietterich, Thomas G and Lathrop, Richard H and Lozano-P{\'e}rez, Tom{\'a}s},
  journal={Artificial intelligence},
  volume={89},
  number={1-2},
  pages={31--71},
  year={1997},
  publisher={Elsevier}
}

@article{cour2011learning_partial,
  title={Learning from partial labels},
  author={Cour, Timothee and Sapp, Ben and Taskar, Ben},
  journal={The Journal of Machine Learning Research},
  volume={12},
  pages={1501--1536},
  year={2011},
  publisher={JMLR. org}
}
\bibliographystyle{colm2026_conference}

\clearpage
\appendix
\section*{Appendix}
\setcounter{equation}{0}

\section{Alternative Similarity Modeling}
\label{app:rbf}
\paragraph{Gaussian.} 
We model the correlation between models using a Gaussian Radial Basis Function (RBF) kernel. This captures the intuition that models with similar embeddings (close in the functional space) are redundant. We define the similarity matrix ${K} \in \R^{M \times M}$ as:
\begin{equation}
    K_{ij} = \exp\left( - \gamma \| {u}_i - {u}_j \|_2^2 \right),
\end{equation}
where ${u}_i, {u}_j$ are the learnable model embeddings and $\gamma > 0$ is a trainable scaling parameter (inverse bandwidth). This formulation guarantees that ${K}$ is positive semi-definite and bounded between $(0, 1]$.

\section{Theoretical Results}
\label{theory_results}
 
\begin{proposition}[Submodularity of Log-Determinant]
\label{thm:submodular}
Let $L_x \succ 0$ be a positive definite kernel matrix over $\mathcal{M}$. Define $f(S) = \log \det((L_x)_S)$ for $S \subseteq \mathcal{M}$, where $(L_x)_S$ denotes the principal submatrix of $L_x$ indexed by $S$, and $f(\emptyset) = 0$ by convention. Then $f$ is submodular: for all $S \subseteq T \subseteq \mathcal{M}$ and $i \notin T$,
\begin{equation}
    f(S \cup \{i\}) - f(S) \geq f(T \cup \{i\}) - f(T).
\end{equation}
However, $f$ is \emph{not} monotone in general. The marginal gain of adding model $i$ to set $S$ is
\begin{equation}
    f(S \cup \{i\}) - f(S) = \log (L_x)_{ii \cdot S},
\end{equation}
where $(L_x)_{ii \cdot S} = (L_x)_{ii} - (L_x)_{i,S}[(L_x)_S]^{-1}(L_x)_{S,i}$ is the Schur complement. This quantity is positive (since $L_x \succ 0$) but may be less than one, in which case $\log(L_x)_{ii\cdot S} < 0$ and $f$ decreases.
\end{proposition}
 
\begin{proof}
For any $S \subseteq \mathcal{M}$ and $i \notin S$, the determinant formula for bordered matrices gives
\begin{equation}
    \frac{\det((L_x)_{S \cup \{i\}})}{\det((L_x)_S)} = (L_x)_{ii} - (L_x)_{i,S} [(L_x)_S]^{-1} (L_x)_{S,i} = (L_x)_{ii \cdot S}.
\end{equation}
For $S \subseteq T$, the Schur complement satisfies $(L_x)_{ii \cdot T} \leq (L_x)_{ii \cdot S}$, since conditioning on a larger set can only reduce the residual variance in a positive semi-definite matrix. Taking logarithms preserves the inequality, establishing submodularity.
 
Since $L_x \succ 0$, all Schur complements are strictly positive, so $f(S \cup \{i\}) - f(S) = \log(L_x)_{ii\cdot S} > -\infty$. However, nothing prevents $(L_x)_{ii \cdot S} < 1$. 
\end{proof}

The non-monotonicity of $f(S) = \log\det((L_x)_S)$ has a natural interpretation in the routing context: adding a model that is highly correlated with the existing selection can reduce the overall determinant, reflecting the fact that redundant coverage does not improve the DPP objective.

\bigskip
 
\begin{proposition}[Greedy Determinant Maximality with Adaptive Stopping]
\label{thm:greedy}
Let $L_x \succ 0$ and let $f(S) = \log\det((L_x)_S)$ as above. Consider the greedy algorithm that initializes $S = \emptyset$ and iteratively selects
\begin{equation}
    i^* = \arg\max_{i \notin S}\; g_i(S) = \arg\max_{i \notin S}\; (L_x)_{ii \cdot S},
\end{equation}
terminating when $g_{i^*}(S) \leq 1$ for the best remaining candidate. Let $S_{\mathrm{greedy}}$ denote the output. Then for all $T \supseteq S_{\mathrm{greedy}}$ with $T \neq S_{\mathrm{greedy}}$,
\begin{equation}
    \det((L_x)_T) \leq \det((L_x)_{S_{\mathrm{greedy}}}).
\end{equation}
That is, $S_{\mathrm{greedy}}$ globally maximizes $\det((L_x)_S)$ over all supersets of itself.
\end{proposition}
 
\begin{proof}
When the algorithm terminates, $g_i(S_{\mathrm{greedy}}) \leq 1$ for all $i \notin S_{\mathrm{greedy}}$, meaning $(L_x)_{ii \cdot S_{\mathrm{greedy}}} \leq 1$. For any superset $T = S_{\mathrm{greedy}} \cup \{i_1, \ldots, i_r\}$, the determinant telescopes as
\begin{equation}
    \det((L_x)_T) = \det((L_x)_{S_{\mathrm{greedy}}}) \prod_{j=1}^{r} (L_x)_{i_j i_j \cdot S_{\mathrm{greedy}} \cup \{i_1,\ldots,i_{j-1}\}}.
\end{equation}
By submodularity of $f$ (Theorem~\ref{thm:submodular}), $(L_x)_{i_j i_j \cdot S_{\mathrm{greedy}} \cup \{i_1,\ldots,i_{j-1}\}} \leq (L_x)_{i_j i_j \cdot S_{\mathrm{greedy}}} \leq 1$, so each factor in the product is at most $1$, giving $\det((L_x)_T) \leq \det((L_x)_{S_{\mathrm{greedy}}})$.
\end{proof}

This shows that the greedy algorithm with adaptive stopping finds a set that no superset can improve upon. In particular, $S_{\mathrm{greedy}}$ is a global maximizer of $\det((L_x)_S)$ over all sets $S \supseteq S_{\mathrm{greedy}}$, including $\mathcal{M}$ itself. 

\bigskip

\begin{proposition}[Coverage-Diversity Equivalence]
\label{prop:coverage-diversity}
Let $C_x \subseteq \mathcal{M}$ denote the correct set and $F_x = \mathcal{M} \setminus C_x$ the failure set. The hit probability decomposes as
\begin{equation}
    P_{\mathrm{succ}}(x) = 1 - P(Y \subseteq F_x) = \sum_{\substack{S \subseteq \mathcal{M} \\ S \cap C_x \neq \emptyset}} \mathcal{P}_{L_x}(S).
\end{equation}
Consequently, minimizing the coverage loss $\mathcal{L}_{\mathrm{hit}}(x)$ is equivalent to maximizing the total DPP probability mass assigned to subsets that contain at least one correct model.
\end{proposition}
 
\begin{proof}
By the law of total probability,
\begin{equation}
    1 = \sum_{S \subseteq \mathcal{M}} \mathcal{P}_{L_x}(S) = \sum_{\substack{S \subseteq \mathcal{M} \\ S \cap C_x \neq \emptyset}} \mathcal{P}_{L_x}(S) + \sum_{\substack{S \subseteq \mathcal{M} \\ S \cap C_x = \emptyset}} \mathcal{P}_{L_x}(S).
\end{equation}
The second sum equals $P(Y \subseteq F_x)$ since $S \cap C_x = \emptyset$ if and only if $S \subseteq F_x$. Therefore,
\begin{equation}
    P_{\mathrm{succ}}(x) = 1 - P(Y \subseteq F_x) = \sum_{\substack{S \subseteq \mathcal{M} \\ S \cap C_x \neq \emptyset}} \mathcal{P}_{L_x}(S).
\end{equation}
Since $\mathcal{L}_{\mathrm{hit}}(x) = -\log P_{\mathrm{succ}}(x)$ and $-\log$ is monotone decreasing, minimizing $\mathcal{L}_{\mathrm{hit}}(x)$ is equivalent to maximizing $P_{\mathrm{succ}}(x)$.
\end{proof}

\section{Experiment Details}

\begin{table}[t]
\centering
\caption{Effect of the relative stopping threshold $\tau$ on coverage and inference cost
  (combined ID test, large-pool setting, $k_{\max} = 10$).
  $\tau$ is expressed as a fraction of the first marginal log-determinant gain.
  Size Reduction and Coverage Drop are relative to $\tau = 0$.}
\label{tab:tau_ablation}
\begin{tabular}{c c c r r}
\toprule
$\tau$ & Avg Subset Size & Success@10 & \!\!Size Reduction & \!\!Coverage Drop \\
\midrule
0.000 & 10.00 & 0.9885 &  0.0\,\% &  0.00\,\% \\
0.001 & 10.00 & 0.9885 &  0.0\,\% &  0.00\,\% \\
0.010 & 10.00 & 0.9885 &  0.0\,\% &  0.00\,\% \\
0.050 & 10.00 & 0.9885 &  0.0\,\% &  0.00\,\% \\
0.100 &  9.95 & 0.9883 &  0.5\,\% &  0.02\,\% \\
0.200 &  6.41 & 0.9772 & 35.9\,\% &  1.13\,\% \\
0.300 &  4.25 & 0.9562 & 57.5\,\% &  3.24\,\% \\
0.500 &  1.41 & 0.8828 & 85.9\,\% & 10.57\,\% \\
\bottomrule
\end{tabular}
\end{table}

\subsection{Benchmarks and Datasets}
\label{app:data}
We conduct experiments on \textbf{RouterEval}~\citep{routereval}, a large-scale benchmark designed for evaluating LLM routing methods. RouterEval aggregates multiple widely used reasoning and knowledge benchmarks and records the performance of thousands of candidate LLMs on each prompt. For every $(\text{query}, \text{model})$ pair, the benchmark provides a binary correctness label indicating whether the candidate model answers the query correctly. This design enables systematic evaluation of routing algorithms under large candidate model pools.

Our experiments cover two RouterEval settings: the \textit{medium-pool} setting and the \textit{large-pool} setting. The medium-pool setting contains \textbf{3811} candidate LLMs, while the large-pool setting contains \textbf{5000} candidate LLMs. Within each setting, all tasks share the same candidate model pool, enabling consistent comparison across routing methods.

For the \textbf{medium-pool setting}, we use three tasks for in-domain training and evaluation: \textbf{BBH}~\citep{bbh}, \textbf{MATH}~\citep{math}, and \textbf{GPQA}~\citep{gpqa}. We evaluate out-of-domain generalization on two additional tasks that are not observed during training: \textbf{IFEval}~\citep{ifeval} and \textbf{MuSR}~\citep{musr}. These tasks cover diverse reasoning settings including multi-step reasoning, mathematical reasoning, graduate-level question answering, instruction-following evaluation, and multi-step reasoning over structured contexts.

For the \textbf{large-pool setting}, the in-domain tasks include \textbf{MMLU}~\citep{hendrycks2021mmlu}, \textbf{HellaSwag}~\citep{zellers2019hellaswag}, \textbf{GSM8K}~\citep{cobbe2021gsm8k}, and \textbf{ARC}~\citep{clark2018arc}. Out-of-domain evaluation is conducted on \textbf{TruthfulQA}~\citep{truthfulqa} and \textbf{WinoGrande}~\citep{sakaguchi2021winogrande}. These datasets span knowledge-intensive reasoning, commonsense inference, mathematical problem solving, and truthfulness evaluation. Table~\ref{tab:routereval_datasets} summarizes the datasets used in our experiments.

Following prior work, we evaluate FlexRouter both in-domain and out-of-domain.
For the in-domain tasks, the router is evaluated on the held-out 20\% test splits of the tasks used for training, measuring its ability to handle unseen instances from familiar tasks.
For the out-of-domain tasks, the router is trained on the in-domain tasks but evaluated on additional tasks that are entirely excluded from training. This setup assesses the router's ability to generalize to previously unseen tasks.

\begin{table*}[t]
\centering
\small
\footnotesize
\scriptsize
\begin{tabular}{lccc}
\toprule
Dataset & Category & \#Prompts & \#LLMs\\
\midrule
BBH & Complex Reasoning & 5761 & 3811 \\
MATH & Mathematical Reasoning & 1324 & 3811 \\
GPQA & Graduate-level QA & 1192 & 3811 \\
IFEval & Instruction Following & 541 & 3811 \\
MuSR & Multi-step Reasoning & 756 & 3811 \\
\midrule
MMLU & Knowledge & 14042 & 5000 \\
HellaSwag & Commonsense Reasoning & 10042 & 5000 \\
GSM8K & Mathematical Reasoning & 1319 & 5000 \\
ARC & Science QA & 1172 & 5000 \\
TruthfulQA & Truthfulness Evaluation & 817 & 5000 \\
WinoGrande & Commonsense Reasoning & 1267 & 5000 \\
\bottomrule
\end{tabular}
\caption{Statistics of datasets used in our work. Each dataset consists of prompts evaluated across a shared pool of candidate LLMs within each setting. 
}
\label{tab:routereval_datasets}
\end{table*}

\subsection{Baselines}
\label{app:baseline}
We compare FlexRouter with several strong routing baselines that represent commonly used model selection strategies using \textbf{Independent scoring}, which ranks candidate models according to predicted correctness scores and selects the top-$k$ models with the highest scores. Each model is evaluated independently, and correlations between models are not considered during selection. This setting represents a common strategy used in many routing and ensemble systems where models are chosen solely based on their estimated accuracy.

\textbf{EmbedLLM}~\citep{embedllm} learns query-dependent scores by jointly embedding queries and candidate models into a shared representation space. The router predicts the likelihood that each model will answer a query correctly based on these embeddings. However, model selection is still performed independently for each candidate, without explicitly modeling correlations between models.

\textbf{BaRP}~\citep{barp} is a learned routing method that models routing as a multi-objective contextual bandit problem. It uses REINFORCE on bandit feedback to learn a routing policy that balances performance and cost when selecting models. BaRP serves as a strong learned baseline for adaptive model selection under resource constraints.

\textbf{Reference model (Ref.\ score).}  
The Ref.\ score corresponds to the performance of a representative high-capacity LLM evaluated directly on the benchmark without routing. This reference model is not a routing method and does not perform model selection. Instead, it reflects the performance of a single strong model such as GPT-4 \citep{gpt4}. The reference score provides a useful point of comparison because it represents the outcome of always relying on a single capable model.

For all baselines, we evaluate performance under identical inference budgets. Each routing method selects the same number of candidate models per query, ensuring that performance differences reflect the quality of the routing strategy rather than differences in computational cost.

\subsection{Evaluation Metrics}
\label{app:metrics}
We evaluate routing performance using metrics that capture both success probability and model diversity.

\textbf{Success@$k$} measures the fraction of queries for which the selected subset contains at least one model that produces a correct answer. This metric directly reflects the routing objective of maximizing the probability that at least one selected model succeeds.

\begin{equation}
  \text{Success@}k
  \;=\;
  \frac{1}{|\mathcal{Q}|}\sum_{q \in \mathcal{Q}}
  \mathbf{1}\!\left[\,\exists\, m \in \mathcal S_q : y_{q,m} = 1 \right].
  \label{eq:success_at_k}
\end{equation}

Here $\mathcal{Q}$ denotes the set of queries, $\mathcal{S}_q$ is the subset of models selected for query $q$, and $y_{q,m} \in \{0,1\}$ indicates whether model $m$ answers query $q$ correctly.

\textbf{ILD@$k$} (Intra-List Diversity) measures the average pairwise cosine distance among the embeddings of the selected models. Larger values indicate lower redundancy and greater complementarity among the selected models.

\begin{equation}
  \text{ILD@}k
  \;=\;
  \frac{1}{|\mathcal{Q}|}\sum_{q \in \mathcal{Q}}
  \frac{1}{\binom{|\mathcal{S}_q|}{2}}
  \sum_{\substack{i,j \,\in\, \mathcal{S}_q \\ i < j}}
  \left(1 - \frac{\mathbf{u}_i^\top \mathbf{u}_j}{\|\mathbf{u}_i\|\,\|\mathbf{u}_j\|}\right),
  \label{eq:ild_at_k}
\end{equation}

where $\mathbf{u}_m$ denotes the embedding of model $m$. The metric ranges from $0$ to $2$, where $0$ indicates that all selected models are identical in the embedding space and $2$ indicates that model embeddings are maximally dissimilar. We compute ILD only for queries with at least two selected models.
To ensure fair comparison across routing methods, the embeddings $\mathbf{u}_m$ are taken from a fixed pre-trained embedding space derived from model metadata. As a result, ILD reflects model diversity independently of each method's internal scoring mechanism.

\subsection{Implementation Details}
\label{app:imp}
FlexRouter uses a shared query encoder together with learnable model embeddings to parameterize the routing policy. 
Queries are encoded using \texttt{RoBERTa-base}, and both query and model representations are projected into a shared latent space with dimension 128. 
The resulting representations are used to construct the query-dependent DPP kernel that defines the subset selection distribution.

The model is trained using the Adam optimizer with a learning rate of $10^{-3}$. 
Training proceeds for up to 100 epochs with early stopping based on validation Success@$k$, using a patience of 15 epochs. The training objective combines the coverage-oriented objective and an auxiliary binary cross-entropy loss. 
The weight for auxiliary supervision is set to $\lambda = 1.0$ for all experiments.

During inference, FlexRouter selects models using a greedy maximum a posteriori procedure based on marginal log-determinant gains of the DPP objective. 
Model selection proceeds iteratively and stops when the marginal gain becomes non-positive or when the maximum subset size $k=10$ is reached. 
This adaptive stopping rule allows the router to adjust the number of selected models according to query difficulty while maintaining a bounded inference budget. All experiments were conducted on NVIDIA A100 80GB GPUs.

While the kernel is of size $M \times M$, greedy inference only requires computing marginal gains with respect to the current subset and maintaining a Cholesky factorization, resulting in $O(k^2 M)$ complexity. In practice, we further restrict evaluation to a top-$K_0$ candidate pool for efficiency.

\section{Additional Results}
The results in this section are from additional runs conducted for further analysis. They follow the same evaluation protocol as the main experiments, but may differ slightly from Tables~1-2 due to stochastic training and rerunning the models. The qualitative trends remain consistent.
\subsection{Additional Coverage Analysis}
\label{app:coverage_analysis}

The main experiments report Success@10 because all routing methods are evaluated under the same maximum budget $k_{\max}=10$. Since RouterEval contains thousands of candidate LLMs in both the medium-pool and large-pool settings, selecting 10 models is still a compact subset of the full candidate pool. To further examine whether the advantage of FlexRouter only appears at $k=10$, we rerun the experiments and report average Success@k across different budgets in Table~\ref{tab:app_success_by_k}.

FlexRouter trails EmbedLLM at $k=1$, which is expected because the method is not optimized for selecting a single strongest model. However, FlexRouter overtakes independent selection once routing becomes genuinely multi-model, e.g., from $k=3$ in the medium-pool setting and from $k=2$ in the large-pool setting. This supports the intended use case of FlexRouter: constructing a compact candidate pool whose selected models cover complementary strengths rather than redundantly selecting similar high-scoring models.

\textbf{Pool composition.} Success@k measures whether the selected pool contains at least one correct model, but it does not show how many correct models are selected or whether improvements come from reducing all-wrong pools. We therefore report two additional pool-composition metrics. Avg-Correct@10 measures the average number of correct models in the selected pool, while Zero-Correct Rate measures the fraction of queries for which all selected models are incorrect. The results are shown in Table~\ref{tab:app_pool_composition}.

The results show a trade-off between correct-candidate density and candidate-pool coverage. EmbedLLM obtains slightly higher Avg-Correct@10, but it also leaves more queries with zero correct candidates. In contrast, FlexRouter reduces the Zero-Correct Rate while maintaining a high number of correct models in the selected pool. Compared with EmbedLLM, FlexRouter reduces the Zero-Correct Rate from 0.160 to 0.128 in the medium-pool setting and from 0.022 to 0.008 in the large-pool setting. This suggests that independent top-k selection may concentrate correct models on some queries while leaving more queries completely uncovered, whereas FlexRouter reduces the number of all-wrong selected pools.

\begin{table}[t]
\centering
\caption{Average Success@k across different model-call budgets. FlexRouter is designed for complementary subset construction rather than single-model routing. It becomes stronger once the setting allows multiple models to be selected.}
\label{tab:app_success_by_k}
\begin{tabular}{llccc}
\toprule
Setting & $k$ & FlexRouter & EmbedLLM & BaRP \\
\midrule
Medium-pool & 1  & 0.417 & 0.542 & 0.527 \\
Medium-pool & 3  & 0.709 & 0.707 & 0.639 \\
Medium-pool & 5  & 0.789 & 0.765 & 0.732 \\
Medium-pool & 7  & 0.832 & 0.805 & 0.784 \\
Medium-pool & 10 & 0.872 & 0.840 & 0.844 \\
\midrule
Large-pool & 1  & 0.700 & 0.750 & 0.665 \\
Large-pool & 2  & 0.928 & 0.925 & 0.822 \\
Large-pool & 5  & 0.975 & 0.956 & 0.878 \\
Large-pool & 10 & 0.992 & 0.978 & 0.916 \\
\bottomrule
\end{tabular}
\end{table}

\begin{table}[t]
\centering
\caption{Pool-composition analysis at $k=10$. Avg-Correct@10 measures the average number of correct models in the selected subset, while Zero-Correct Rate measures the fraction of queries for which the selected subset contains no correct model.}
\label{tab:app_pool_composition}
\begin{tabular}{llccc}
\toprule
Setting & Method & Avg-Correct@10 $\uparrow$ & Zero-Correct Rate $\downarrow$ & Success@10 $\uparrow$ \\
\midrule
Medium-pool & EmbedLLM   & 5.17 & 0.160 & 0.840 \\
Medium-pool & BaRP       & 4.51 & 0.156 & 0.844 \\
Medium-pool & FlexRouter & 4.90 & 0.128 & 0.872 \\
\midrule
Large-pool & EmbedLLM   & 7.94 & 0.022 & 0.978 \\
Large-pool & BaRP       & 6.64 & 0.084 & 0.916 \\
Large-pool & FlexRouter & 7.49 & 0.008 & 0.992 \\
\bottomrule
\end{tabular}
\end{table}

\subsection{Coverage-Oriented Baselines}
\label{app:coverage_baselines}

The main experiments compare FlexRouter with strong routing baselines, including EmbedLLM and BaRP. However, these methods mainly score candidate models independently and are not explicitly designed to optimize Success@k. To further test whether FlexRouter's gains come from the DPP coverage objective rather than from adding any diversity heuristic, we compare with three additional coverage- and diversity-oriented baselines.

\paragraph{Random-k.}
Random-k uniformly selects $k$ models from the candidate pool.

\paragraph{MaxDiversity.}
MaxDiversity greedily selects models to maximize pairwise embedding diversity. We initialize the selected set with the highest predicted-quality model and then greedily add the model that maximizes the minimum pairwise distance to the already selected models.

\paragraph{MMR.}
MMR greedily balances predicted model quality and redundancy. At each step, it selects
\begin{equation}
i^* = \arg\max_{i \notin S}
\alpha q_i(x) - (1-\alpha)\max_{j\in S}\mathrm{sim}(i,j),
\end{equation}
where $q_i(x)$ is the predicted quality score and $\mathrm{sim}(i,j)$ is cosine similarity in the fixed model-embedding space. This baseline is closest to a diversity-aware top-k heuristic, since it applies a post-hoc pairwise penalty to independently predicted quality scores.

\begin{table}[t]
\centering
\caption{Coverage-oriented baselines across medium- and large-pool settings. FlexRouter outperforms Random-k, MaxDiversity, and MMR across different budgets.}
\label{tab:app_coverage_baselines}
\begin{tabular}{llcccc}
\toprule
Setting & $k$ & FlexRouter & MMR & MaxDiversity & Random-k \\
\midrule
Medium-pool & 3  & 0.709 & 0.706 & 0.689 & 0.592 \\
Medium-pool & 5  & 0.789 & 0.757 & 0.754 & 0.697 \\
Medium-pool & 10 & 0.872 & 0.842 & 0.837 & 0.806 \\
\midrule
Large-pool & 3  & 0.954 & 0.921 & 0.915 & 0.745 \\
Large-pool & 5  & 0.975 & 0.942 & 0.935 & 0.812 \\
Large-pool & 10 & 0.992 & 0.961 & 0.950 & 0.874 \\
\bottomrule
\end{tabular}
\end{table}

FlexRouter consistently outperforms the coverage-oriented baselines in both RouterEval settings. The comparison with MMR is especially informative: MMR applies diversity as a post-hoc penalty to independently predicted quality scores, while FlexRouter jointly learns query-dependent model quality and redundancy through the DPP coverage objective. These results suggest that the improvement is not simply due to adding a generic diversity heuristic, but comes from learning complementarity under the coverage-oriented subset objective.

\begin{table}[t]
\centering
\caption{Supervision ablation in the medium-pool setting. FlexRouter remains effective when trained with fewer labeled queries.}
\label{tab:app_supervision_ablation}
\begin{tabular}{lcccccc}
\toprule
Training Queries & MATH & BBH & GPQA & IFEval & MuSR & Avg Success@10 \\
\midrule
25\%  & 0.785 & 0.943 & 0.836 & 0.963 & 0.772 & 0.860 \\
50\%  & 0.793 & 0.941 & 0.845 & 0.954 & 0.792 & 0.865 \\
100\% & 0.793 & 0.946 & 0.878 & 0.963 & 0.838 & 0.884 \\
\bottomrule
\end{tabular}
\end{table}

\subsection{Supervision Ablation}
\label{app:supervision_ablation}

RouterEval provides dense binary correctness labels for each query--model pair. To study how much query-level supervision FlexRouter needs, we train on 25\%, 50\%, and 100\% of the available training queries in the medium-pool setting, while evaluating all checkpoints on the same held-out test set. For each retained query, we keep all model-level correctness labels. Therefore, this ablation measures query-level supervision efficiency rather than model-level label sparsity.

FlexRouter degrades gracefully as the number of labeled training queries decreases. With only 25\% of the training queries, it achieves 0.860 average Success@10, within 2.4 percentage points of the full-supervision setting. With 50\% of the training queries, it achieves 0.865 average Success@10, within 1.9 percentage points of the full-supervision setting. This suggests that FlexRouter can learn useful quality-diversity structure without using the full query pool.

This experiment studies query-level label sparsity. It does not address model-level label sparsity, where only a subset of candidate models is labeled for each query. That setting is closer to matrix completion or bandit-style exploration and remains an important direction for scalable routing.

\section{Related Work}
\paragraph{LLM Routing.}
Routing across LLMs has emerged as a practical approach for improving performance under heterogeneous model capabilities and costs. A common strategy is to predict model-wise correctness scores and select the top-$k$ models independently. Methods such as EmbedLLM~\citep{embedllm} learn joint representations of queries and models to estimate query-dependent performance. Other recent approaches, such as RouterDC~\citep{chen2024routerdc} and BaRP~\citep{barp}, formulate routing as representation learning or contextual bandit problems. However, these methods typically score models independently and do not explicitly account for correlations or shared failure modes, which can lead to redundant selections.

\paragraph{Multi-Model Ensembles and Verification.} Our coverage-oriented objective aligns with multi-generation LLM pipelines where a downstream verifier or ranker selects the final answer. This paradigm is widely used in ensemble methods like LLM-Blender~\citep{jiang2023llm_blender}, which ranks outputs from multiple diverse models, and self-consistency frameworks~\citep{wang2022self_consistency} that aggregate multiple reasoning paths. 
More recently, frameworks like Mixture-of-Agents~\citep{wang2024mixture_of_agents} have demonstrated that layering multiple LLMs can significantly enhance generation quality through collaborative refinement.
While these methods benefit from diverse candidate pools, they typically rely on fixed, manually curated model sets. FlexRouter automates and optimizes the construction of these complementary candidate sets on a per-query basis.

\paragraph{Determinantal Point Processes.}
Determinantal Point Processes~\citep{kulesza2012dppfnt} have been widely used for subset selection tasks that require balancing quality and diversity, including document summarization~\citep{cho2019multi_document}, recommendation~\citep{wilhelm2018practical_recommendation}, and information retrieval~\citep{affandi2014learning_information_retrieval, deng2020personalized_information_retrieval}. A DPP defines a distribution over subsets where the determinant of a kernel matrix captures both item quality and pairwise diversity. Prior work has explored learning DPP kernels from data and applying DPPs to subset selection problems with fixed ground-truth subsets. In contrast, our setting involves weak supervision where no single optimal subset is observed. We therefore design a coverage-oriented objective that directly optimizes the probability of selecting at least one correct model, while leveraging the diversity-inducing properties of DPPs to model complementarity among LLMs.

\end{document}